\documentclass[onecolumn]{IEEEtran}
\usepackage{amsmath}
\usepackage{dsfont}
\usepackage{multirow}
\usepackage{amsthm}
\usepackage{xspace}
\newcommand{\agent}{SARL\xspace}
\newtheorem*{theorem}{Theorem}

\ifCLASSINFOpdf
   \usepackage[pdftex]{graphicx}
\else
\fi

\begin{document}
%
\title{A Socially Aware Reinforcement Learning Agent for The Single Track Road Problem}

\author{\IEEEauthorblockN{Ido Shapira and Amos Azaria}\\
\IEEEauthorblockA{Computer Science Department\\
Ariel University
}
}

\maketitle

\begin{abstract}
We present the single track road problem. In this problem two agents face each-other at opposite positions of a road that can only have one agent pass at a time. We focus on the scenario in which one agent is human, while the other is an autonomous agent. We run experiments with human subjects in a simple grid domain, which simulates the single track road problem.
We show that when data is limited,  building an accurate human model is very challenging, and that a reinforcement learning agent, which is based on this data, does not perform well in practice. However, we show that an agent that tries to maximize a linear combination of the human's utility and its own utility, achieves a high score, and significantly outperforms other baselines, including an agent that tries to maximize only its own utility. 
\end{abstract}


%
\IEEEpeerreviewmaketitle

\section{Introduction}

While humans can cope with new situations quite easily, even state-of-the-art algorithms struggle with new situations that they haven't been trained on. Unfortunately, when it comes to autonomous vehicles the results may be devastating. One example for an uncommon, yet important scenario for autonomous vehicles is the problem of a single track road. In this problem two vehicles in opposite directions must cross a narrow road, which is not wide enough to allow both vehicles to pass at the same time. Therefore, one vehicle must deviate from the road and let the other vehicle cross. However, if both vehicles deviate to the margins, they might both return to the road, and may either end-up deviating to the margins again or even colliding with each-other.
Despite only a small portion of the roads being single track roads, autonomous vehicles must be able to function properly in these types of roads. Furthermore, some more common situations resemble the single track road problem, for example, if cars park where they shouldn't and block one of the lanes or if one lane is blocked for any other reason (e.g., a falling tree), the traffic in both ways must operate with a single lane.

In this paper we model the single track road problem as a sequential two player game on a two row grid (see Figure \ref{fig:game}). The upper row represents a road that allows both players to advance. However, the lower row can only be used for allowing the other player to pass, as the players cannot advance when placed in the lower row.
We find several equilibria of the game, which should determine how a perfectly rational agent should behave in such a game.
However, people tend to deviate from what is considered rational behavior, since they are influenced by different effects including anchoring, inconsistency of utility and a lack of understanding of other agent's behavior~\cite{tversky81,ArielyAnchor,camerer03}. Indeed, as we later show, while some people tend to follow the game theoretic solution, many others do not follow it, and behave unexpectedly.
Due to non-perfectly rational behavior of humans, algorithmic approaches that assume rational behavior tend to perform poorly with humans~\cite{bitan2013social,azaria2015strategic,nay2016predicting}.

\begin{figure}[t]
\label{fig:game}
\centering
\includegraphics[width=4in]{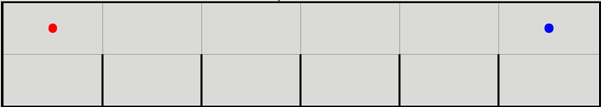}
\caption{The initial state of the single road game board. The red circle is controlled by the human player and the blue circle is controlled by the autonomous agent. Both players must reach the opposite side of the board without colliding.
The players may travel freely on the upper row, but they cannot advance when located on the lower row.}
\end{figure}

Therefore, a common approach for developing an agent that can proficiently interact with humans is composed of several stages \cite{azaria2011giving,nguyen2013analyzing,rosenfeld2015adaptive}. The first stage includes the collection of a data-set of humans interacting in the environment. Next, based on the collected data-set a human behavior model is developed, usually by applying machine-learning techniques. Finally, the human model is used by the agent to determine the actions that are the most beneficial for it.  
In this paper we attempt to follow this common practice for the single track road game. Therefore, we collect human data in this game and use it to compose a human model. Then, we model the agent's problem as a reinforcement learning environment by an MDP with the human model being a part of the environment. Finally, we use value iteration, a dynamic programming based method, to find the supposedly optimal action for the agent. We note that the solution provided by value iteration is guaranteed to be optimal under the assumption that the MDP models the environment perfectly, which includes the human model. 

However, composing a human behavior model based on a relatively small data-set may be inaccurate, as people are many times unpredictable and different humans tend to behave very differently from one another, despite a game being relatively simple \cite{shvartzon2016personalized,azaria2016autonomous}. Therefore, we introduce a novel method for solving an MDP that is based on a non-perfect human behavior model. Namely, instead of solving the MDP with an attempt to maximize the agent's outcome, we propose to maximize a linear combination of the agent's outcome and the human's outcome. We expect that optimizing toward a linear combination will be beneficial for the agent, since the humans are likely to try and optimize their own utility function, so they are likely to deviate from the human model in a way that will indeed maximize their utility function. By optimizing toward a linear combination, the agent acts as if it already accounts for these deviations and is therefore more likely to adapt to them.
We provide a formula for determining the proposed linear combination, which is based on the similarity of the agents' utility functions.

We introduce our Socially Aware Reinforcement Learning agent (\agent), an agent that attempts to maximize the linear combination of the two utility functions, using our proposed formula. We show that \agent significantly outperforms all other baselines in the single track road game when interacting with humans, in terms of the agent's final outcome. Somewhat less surprisingly, humans interacting with \agent also achieve the highest outcome. Therefore, \agent not only perform better with respect to its own outcome, but also with respect to social welfare.

To summarize, the contributions of this paper are three-fold:
\begin{enumerate}
    \item We present the single track road problem, model it as a sequential game, and present the equilibria of the game.
    We show that people do not follow strategies that are in equilibrium.
    \item We model the problem as an MDP in which the human's actions are modeled as a part of the environment. The model uses data from humans interacting with simple agents to determine the probability of the human taking each action at a given state. 
    \item We present \agent, a socially aware reinforcement learner, that uses a linear combination of the rewards of both agents. We provide a formula for finding the parameter to be used in this linear combination. Finally, we show that our method significantly outperforms all other baselines.
\end{enumerate}

\section{Related Work}
Trajectory prediction of surrounding vehicles and pedestrians is very important for the development of autonomous vehicles, as such knowledge can prevent accidents. Indeed, trajectory prediction is challenging due to the unexpected nature of human behavior.
Therefore, many works attempt to find a sufficient solution to overcome this challenge \cite{leon2021review}.

According to Houenou et al. \cite{AP} trajectory prediction can be based on a deterministic method that selects the current maneuver from a predefined set using kinematic measurements and road geometry detection.
The authors state that their model cannot be applied to very low speed scenarios and therefore is not applicable to our scenario.

Deo and Trivedi \cite{DT} estimate a probability distribution of future positions of a vehicle conditioned on its track history and the track histories of vehicles around it, at a certain time. Using this information, they select one of six possible maneuvers that have been defined.
They use the publicly available NGSIM US-101
and I-80
highway data-sets for their experiments. Their model relies purely on vehicle tracks to infer maneuver classes and ignores the lanes and the map.

Ding at al. \cite{DC} use a {\it recurrent neural network} for composing an observation encoding. Based on this encoding, they propose a {\it Vehicle Behavior Interaction Network} (VBIN) to capture the social effect of another agent on the prediction target, based on their maneuver features and relative dynamics (e.g., relative positions and velocities).
VBIN is an end-to-end trainable framework and is suitable for dynamic driving scenarios where the dynamics of the agents affect their importance in social interactions. They use data collected from highways US-101 and I-80 as used in other work \cite{DT}; since it deals only with highway roads with a large number of agents, it is not applicable for our setting.

Kim et al. \cite{KM} propose a deep learning approach for trajectory prediction based on a {\it Long Short Term Memory} (LSTM). Their model is used to analyze the temporal behavior and predict the future coordinates of the surrounding vehicles. Based on the coordinates and velocities of the surrounding vehicles, the vehicle’s future location is produced after a certain short amount of time.
However, the experiments were conducted using data collected from highway driving, which is again not suitable to our case.

Chandra at al. \cite{CG} present an approach for trajectory prediction in urban traffic scenarios using a two-stream graph-LSTM network. The first stream predicts the trajectories, while the second stream predicts the behavior (i.e. overspeeding, underspeeding, or neutral) of road-agents.
It is based on the vehicle coordinates and a weighted dynamic geometric graph (DGG) that represents the relative proximity among road agents.
They also present a rule-based behavior prediction algorithm to forecast whether a road agent is overspeeding
(aggressive), underspeeding (conservative), or neutral,
based on the traffic behavior classification from the psychology literature. They evaluate their approach on the Argoverse, Lyft, Apolloscape, and NGSIM datasets and highlight the benefits over prior trajectory prediction methods.

Elhenawy at al. \cite{EM} introduce a real time game-theory-based algorithm that is inspired by the {\it chicken-game} for controlling autonomous vehicle movements at uncontrolled intersections. They assume that all vehicles communicate to a central management center in the intersection to report their speed, location and direction. The intersection management center uses the information from all vehicles approaching the intersection and decides which action each vehicle will take. They further assume that vehicles obey the {\it Nash-equilibrium} solution of the game and will take the action received from the management center. Unfortunately, these assumptions are very strong and cannot be applied to our setting.

Camara at al. \cite{CF} suggest a more realistic game-theory model based on the {\it sequential chicken-game}.
The model assumes both agents share the same parameters $U_{crash}$ and $U_{time}$, both know this is the case, and both play optimally from their state. It assumes that no lateral motion is permitted, and that there is no communication between the agents other than seeing each other’s positions. 
The sequential chicken-game can be viewed as a sequence of one-shot (sub-)games, which can be solved similarly.
The sub-game at time $t$ can be written as a standard game theory matrix, which can be solved using recursion, game theory, and equilibrium selection to give values and optimal strategies at every state. While they handle the case of a junction by finding a Nash equilibrium and assuming that humans obey it, we deal with the single track road and give not only a game-theory analysis but also provide a novel Reinforcement Learning solution that does not require assumptions about humans and Nash equilibria.




There have been several previous works attempting to model human behavior in normal form games \cite{wright2010beyond,wright2014level}. Wright and Leyton-Brown \cite{wright2010beyond} collected the results of multiple experiments from normal form games studied in the literature, and showed how the human action distribution can be modeled with high accuracy.
However, our problem is clearly more complex and cannot be modeled as a simple normal form game.

Azaria el al. \cite{azaria2012strategic,azaria2016strategic} introduce SAP, a social agent for advice provision. They show that humans tend to ignore advice provided by a selfish agent. Therefore, they suggest using some linear combination of the user's and the agent's preferences. The exact ratio is determined by simulating human behavior and selecting the ratio that achieve the highest performance for the agent in simulation. Therefore, both SAP and our work attempt to maximize agent performance and consider a linear combination of both the user and the agent, however, the environment and settings are completely different, as SAP is an agent for advice provision, and we use a grid environment. In addition, the purpose of the linear combination used by SAP is to address the issue of human trust, while in our work, it is used to mitigate the uncertainty we have in our human model.
Furthermore, we propose a formula for obtaining our proposed ratio, rather than running a simulation for obtaining that value.

\section{The Single Track Road Game}
We now provide a formal definition for the single track road game, which is the main focus of this paper.
Two agents $A$ and $B$ are placed on a $2 \times n$ grid at both ends on the upper row, where agent $A$ is positioned at the upper right corner, with coordinates $(1,n)$, and agent $B$ is positioned at the upper left corner, with coordinates $(1,1)$, each agent's goal is to maximize $u(W)$, their future outcome where $W$ refers to the agent. Each agent's goal is to reach the other side in a minimal number of steps, and without colliding with the other agent.
The set of actions available for each agent depends on its location. In the upper row each agent can perform the following actions: 
\begin{itemize}
    \item {\it Advance}: move toward the other side.
    \item {\it Stay}: remain in current position.
    \item {\it Down}, move to the bottom row.
\end{itemize}
In the bottom row each agent can perform one of the following actions:
\begin{itemize}
    \item {\it Stay}: remain in current position.
    \item {\it Up}: return to the top row.
\end{itemize}
Both agents take actions synchronously, and do not observe the other's action before they take their own action.
We define the reward function as follows:
\begin{itemize}
    \item {\it Collusion}: if both agents collide, each agent loses 100 points, and the game ends.
    \item {\it Arrived at destination}: an agent that arrives at its destination receives a reward of 30 points. The game ends only for the agent that has reached its destination, i.e., the second agent continues to play until it reaches its destination, in which case it will receive a reward of 30 points as well.
    \item {\it Time loss}: any agent that is still in the game (did not reach its destination or collided with the other agent) loses 1 point each time-step.
\end{itemize}

\section{Game Theoretical Analysis}
\label{sec:gameTheory}
In this section we present the game-theory analysis for the single track road problem. Let $x(W)$ be the $x$ coordinate (column) of agent $W$ and let $y(W)$ bet its $y$ coordinate (row).
Let $d(A,B) = x(A) - x(B)$. Note that if agent $B$ has passed agent $A$, $d(A,B)$ will be negative. 



\begin{theorem}
\label{thr:eq}
For two agents $A,B$ in the $2 \times n$ grid of the single track road game. The following strategies are in a sub-game perfect Nash equilibrium:
\begin{itemize}
\item Agent $A$ uses the following strategy:
\begin{itemize}
    \item If $y(A)=1$ (it is in the upper row) it takes action {\it Advance}.
    \item If $y(A)=2$ (it is in the lower row) it takes action {\it Up}.
\end{itemize}
\item Agent $B$ uses the following strategy:
\begin{itemize}
\item If $y(B)=1$ (the agent is in the upper row):
    \begin{itemize}
        \item If $d(A,B) \geq 3$ or $d(A,B) < 0$, it takes action {\it Advance}.
        \item If $y(A)=1$ and $d(A,B) = 1$ it takes action {\it Down}.
        \item If $y(A)=1$ and $d(A,B) = 2$, it may either take action {\it Stay} or {\it Down} (or any mixed strategy of the two).
    \end{itemize}    
\item If $y(B)=2$ (the agent is in the lower row):
    \begin{itemize}
        \item If $d(A,B) \leq 0$ it takes action {\it Up}.
        \item If $y(A)=1$ and $d(A,B) = 1$ it takes action {\it Stay}.
        \item If $y(A)=1$ and $d(A,B) \geq 4$ it takes action {\it Up}.
        \item If $y(A)=2$ and $d(A,B) \geq 3$ it takes action {\it Up}.
        \item Otherwise, it may either take action {\it Stay} or {\it Up} (or any mixed strategy of the two).
    \end{itemize}
\end{itemize}
\end{itemize}
\end{theorem}
\begin{proof}
The proof handles each of the agents separately and shows that no agent should deviate from its determined strategy under the assumption that the other agent remains with its strategy. This is true also for any sub-game.
Given agent $B$'s strategy, agent $A$ should not deviate, as deviation will either cause it longer to reach its destination (resulting in a lower reward), or to collide with agent $B$ (if it decides to take action {\it Down} when agent $B$ is directly below it), resulting in a much lower reward.
Similarly, given agent $A$'s strategy, agent $B$ should not deviate, due to the following:
\begin{itemize}
    \item If $y(B) = 1$ (the agent is in the upper row):
        \begin{itemize}
            \item If $y(A) = 1$ and $d(A,B) = 1$, under the assumption that $A$ would {\it Advance}, taking an action other than {\it Down} would lead to a collision, which will result in a very low reward.
            \item If $d(A,B) \geq 3$ or $d(A,B) < 0$, so either agent $A$ is very far or it has already passed agent $B$. Therefore, there is no risk of collision, and deviating and taking action {\it Down} or {\it Stay} will result in arriving later at the destination, which will result at a lower reward.
            \item If $y(A) = 1$ and $d(A,B) = 2$, deviating and taking action {\it Advance} would result in a collision. Therefore, agent $B$ should take either action {\it Down} or {\it Stay} (or any mixed strategy of the two).
        \end{itemize}
    \item If $y(B) = 2$ (the agent is in the lower row):
        \begin{itemize}
            \item If $d(A,B) < 0$, there is no risk of a collision since agent $A$ already passed agent $B$. Therefore, deviating and playing {\it Stay} delays $B$'s arrival at the destination.
            \item If $y(A) = 1$ and $d(A,B) = 1$, playing action {\it Up} (instead of {\it Stay}) will lead to a collision, resulting in a lower reward.
            \item If $y(A) = 1$ and $d(A,B) \geq 4$, since there is no risk of collision, taking action {\it Up} will yield the greatest reward, and any other action will cause it to reach the destination later.
            \item If $y(A) = 2$ and $d(A,B) \geq 3$, similarly, any action other than {\it Up} will cause a delay in arriving at the destination.
            \item Otherwise, agent $B$ can choose whether to take action {\it Stay} or {\it Up} because there is no risk of a collision and it will not affect the arrival time. We note that if it takes action {\it Up} and agent $A$ follows its strategy, agent $B$'s next action will be {\it Down}.
        \end{itemize}
\end{itemize}


\end{proof}

Clearly, due to the symmetry of the game, agents $A$ and $B$ may switch policies and the resulting set of strategies will be in equilibrium. 
However, since both sets of policies and equilibria are symmetrical, we cannot predetermine which equilibrium to select.
Furthermore, as we will show in the experiments, human agents, in most cases, do not follow any of the above strategies (see section \ref{sec:R}).

\section{Socially Aware Reinforcement Learning (\agent)}

To solve the single track road problem, we introduce the Socially Aware Reinforcement Learning agent (\agent). For \agent, we model the problem as an MDP, in which the human is a part of the environment. A state is composed of the current location of both agents, as well as the location of both agents in the previous timestep, which serves as a model of the velocity. In order to model human behavior, we use a data-set of humans interacting in the single track road problem, and for each state we compute the fraction of humans that were in that state and took each of the possible actions. We apply the Laplace rule of succession \cite{zabell1989rule}. Given state $s$, let $$A_s = \{a \in A: a \text{permitted from} s\}$$
where $A$ is the set of all actions.
For $a \in A_s$ let $|a_s|$ be the number of times in the dataset that action $a$ was performed from state $s$ and let $n_s = \sum_{a \in A_s} |a_s|$ be the total number of actions performed from $s$. We assume that the probability that the human will take action $a$ at state $s$ is
$$\mathds{P}(a | s) = \frac{|a| + 1}{n_s + |A_s|}.$$ 

We use value iteration, an MDP planning algorithm, which is based on dynamic programming, for solving the MDP \cite{sutton2018reinforcement}. The value iteration guarantees to find the optimal policy, under the assumption that the MDP model, which includes the human model, is accurate. Clearly, policy iteration would have yield the same policy, and model free reinforcement learning methods, such as q-learning (using a simulation), should also converge to the same policy, if allowed enough running time.
However, since our human model is inaccurate, instead of using the common approach for solving the MDP by trying to maximize the agent's outcome directly, \agent uses a linear combination of its own outcome and the human's outcome. It is important to note that \agent is still selfish, it considers the human's outcome only because this is its way to maximize its own outcome. It is interesting to note that it has been shown in the field of psychology that people who consider other people's goals and show empathy, feel better with themselves and are more likely to reach their own goals \cite{carey2021deconstructing}.


However, since our human model is based on a limited data-set size, we propose to incorporate our knowledge related to the human reward function into the optimization problem. Yet, instead of trying to use the human reward function as a part of the human model or the transition function of the MDP, we propose to add it to the objective function of the agent.
To that end, we define the parameter $\beta$, a value between $0$ and $1$, that quantifies the degree to which the agent considers its own outcome and the human's outcome. Namely, the agent, $A$, instead of optimizing towards $u(A)$, optimizes towards $\beta u(A) + (1-\beta)u(B)$. We note that when $\beta = 1$ the agent optimizes towards its own outcome. A $\beta$ value of $0.5$ denotes that the agent tries to optimize the social outcome (i.e., $0.5 u(A) + 0.5 u(B)$, which is identical to optimizing simply towards $u(A) + u(B)$), and when $\beta = 0$ the agent only considers the human's utility function. 

In general, for a two player game, in which one of the players is human, and given a data-set we propose a formula for computing the $\beta$ value as follows. Let the vectors $R_A$ and $R_B$, of length $n$, denote the final outcomes in the data-set for players A and B, respectively for each episode.
The following formula provides the proposed $\beta$ value to be used by \agent:
$$\beta = \frac{1 - correl(R_A,R_B)}{2}.$$
$correl(R_A,R_B)$ is the correlation between $R_A$ and $R_B$, which is computed by:
$$correl(X, Y) =  \frac{\sum \limits_{x_i\in X, y_i\in Y} (x_i - \bar{x})(y_i - \bar{y})}{\sqrt{\sum \limits_{x_i \in X}(x_i - \bar{x})^2\sum \limits_{y_i \in Y}(y_i - \bar{y})^2}}.$$
For example, in a zero-sum game, the correlation between the rewards of both players is $-1$; therefore, the agent will ignore the human's outcome and only maximize its own. On the other hand, when the rewards of both players are independent, the correlation will be $0$. Therefore,  $\beta$ will be $0.5$, that is, the agent will consider both rewards equally. In a game that is more cooperative the human's utility function is not as different from the agent's, and therefore the value of $\beta$ is lower.

\section{Problem Specification}
\label{sec:ED}
We use a $2 \times 6$ grid to model the single-road game problem, and the reward functions described in Section \ref{sec:gameTheory} (see Figure \ref{fig:game}). We set the discount factor, $\gamma$, to $0.999$, so that the overall return is very close to the sum of the rewards. 



We define a state as a pair $(i,j)$ in which $i$ is a position of the autonomous agent, and $j$ is a position of the human agent. We refer to this state representation as a state \emph{without} velocity.
We also use a more complex representation of a state by considering also the previous locations of both players; this representation is referred to as a state with velocity. That is, a state is a tuple of two pairs $((i,j),(l,k))$, where the first coordinate of each pair corresponds to the position on the board of the autonomous agent, and the second coordinate corresponds to the position of a human agent. The first pair, $(i,j)$, is the current state of the two agents, and the second pair, $(l,k)$, is their previous state. Accounting for both players' velocity allows the composition of a more accurate human model, which is a part of the transition function.

Since the game ends only when both players reach their destination (or collide with each other), in order to model the problem as an MDP with an objective function that also considers the human's reward function we had to slightly modify the reward function. This is because a standard MDP only considers the agent's point of view while the human is considered  a part of the environment.
Given $\beta$ and a state $s$, let $remainingSteps(s,X)$ be the number of remaining steps from state $s$ for agent $X$, where $X$ is either the autonomous agent or human player.
\begin{itemize}
    \item If the two agents collide, the reward is
    $$\beta \cdot (-100) + (1-\beta) \cdot (-100).$$
    \item If the autonomous agent has not yet reached the destination but the human agent has, the reward is
    $$\beta \cdot (30 - remainingSteps(s, autonomous)) + (1-\beta) \cdot 30.$$
    \item If the autonomous agent has reached the destination and the human agent has not, the reward is
    $$\beta \cdot 30 + (1-\beta) \cdot (30 - remainingSteps(s, human)).$$
    \item For any other step, the reward is 
    $$\beta \cdot (-1) + (1-\beta) \cdot (-1).$$
\end{itemize}

\section{Experimental Design}
In order to compose the data-set and evaluate \agent's performance, we recruited $470$ participants from Mechanical Turk \cite{paolacci2010running} to play the single road game.
The participants first read the game instructions and were then required to answer three short and simple questions, to ensure that they had read and understood the instructions. The participants then played the game only once. Upon completion (either by reaching the other side, or if colliding with the other agent), the participants provided demographic information including whether they have a valid driving license, an expired driving license or no driving license. In addition, the participants
were asked to state how much they agreed with each of the following five statements:
\begin{enumerate}
    \item The agent played aggressively.
    \item The agent played generously.
    \item The agent played wisely.
    \item The agent was predictable.
    \item I felt the agent was a computer.
\end{enumerate}
We used a seven point Likert-like scale \cite{joshi2015likert} for these statements, ranging from strongly disagree ($1$) to strongly agree ($7$).

$446$ participants completed the game and answered the survey. Their demographic information can be found in Table \ref{tbl:participants}.
\begin{table}
\centering
\caption{Participants details}
\begin{tabular}{|c|c|c|} \hline
\multirow{2}{*}{\textbf{Gender}}
& Male & 267 \\
& Female & 179 \\ \hline
\multirow{4}{*}{\textbf{Education level}}
& High school & 128\\
& BA & 265 \\
& M.Sc & 45 \\
& Ph.D & 7 \\
& None & 1 \\ \hline
\multirow{2}{*}{\textbf{Driving License}} 
& Valid & 411 \\
& Expired & 13 \\
& None & 22 \\ \hline
\textbf{Average age} &  & 38.6  \\ \hline
\end{tabular}
\label{tbl:participants}
\end{table}


We used the following 4 different baseline agents for the data gathering phase. 
\begin{enumerate}
    \item 
    {\it Careful:} an agent that adheres to the strategy of agent $B$ in Theorem \ref{thr:eq}. That is, it tries to move left, but tries to avoid colliding with the other agent as well, so if moving left may risk colliding with the other agent it stays in place. If staying in place also risks colliding with the other agent, it moves down.  
    \item
    {\it Aggressive:} an agent that adheres to the strategy of agent $A$ in Theorem \ref{thr:eq}. That is, the agent always moves left.
    \item
    {\it Semi-aggressive:} an agent that moves left unless the other agent is already there, in which case it stays in place until the other agent moves out of its way.
    \item
    {\it Random:} an agent that moves randomly.
\end{enumerate}

\section{Results}
\label{sec:R}

In this section we present a comparison of all agents mentioned above and show that \agent significantly outperforms all other agents.
In addition, we consider the following agents:
\begin{enumerate}
    \item {\it Non-Velocity VI}: runs a value iteration on the MDP without velocity using the appropriate human model.
    \item {\it Velocity VI}: runs a value iteration on the MDP with velocity using the appropriate human model.
    \item {\it Equal Social VI}: uses value iteration to maximize the sum of the agent's and the human's utilities (i.e., used a $\beta$ value of $0.5$).
\end{enumerate}
The agent's score is calculated by averaging all its scores in each game it plays. 
We begin by comparing the performance of each of the agents. 
Figure \ref{fig:AS} presents a comparison between the performance of all baseline agents, Velocity and Non-Velocity Value Iteration ($\beta = 1$), Equal Social Value Iteration ($\beta = 0.5$) and \agent ($\beta = 0.13$).
\begin{figure}[!h]
\centering
\includegraphics[width=4in]{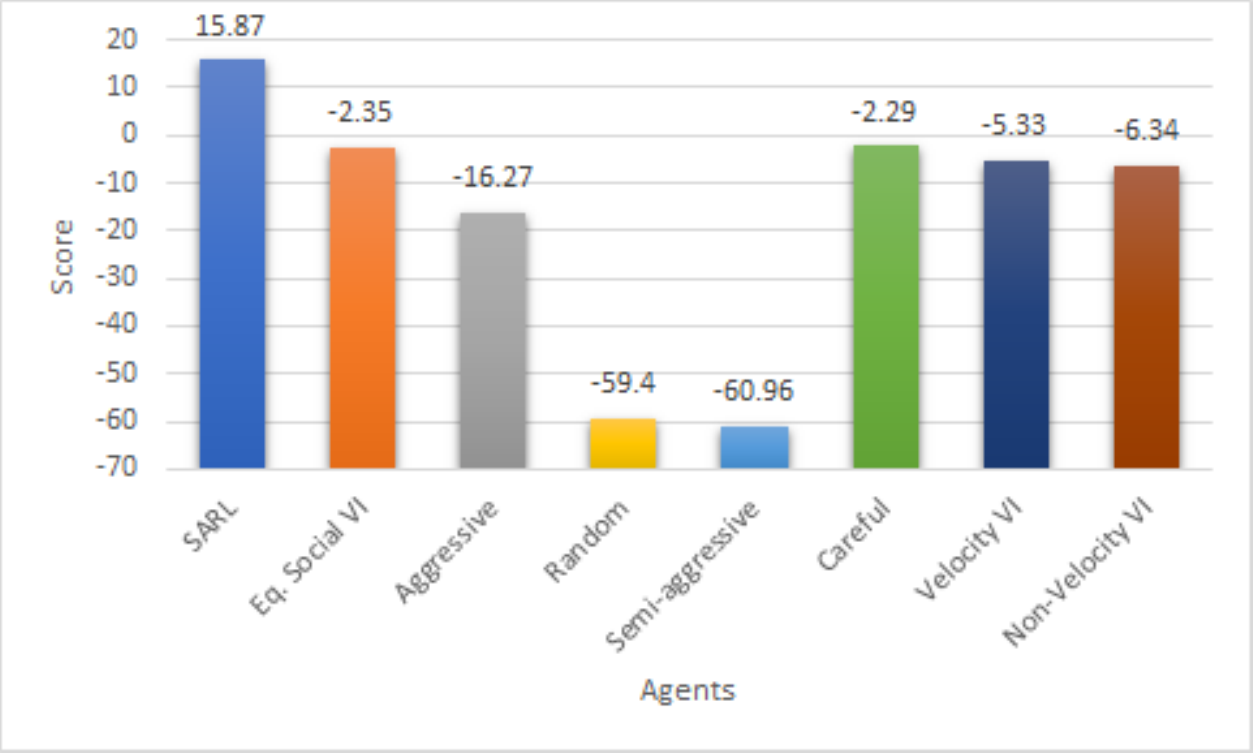}
\caption{A comparison between the performance of all baseline agents, Velocity and Non-Velocity Value Iteration ($\beta = 1$), Equal Social Value Iteration ($\beta = 0.5$) and \agent ($\beta = 0.13$).}
\label{fig:AS}
\end{figure}
As depicted by Figure \ref{fig:AS}, \agent significantly outperforms all other agents ($p < 0.01$) in terms of the agent's performance.  Furthermore, it is the only agent that achieved a positive average reward.
We also note that the agent that uses the state representation with velocity obtained slightly better results than the agent that used the non-velocity state representation, though these differences are not statistically significant.
We now turn to evaluate the human's score when playing with each of the agents. Although the agents are designed to be selfish, clearly, it is more beneficial if also the human player would result with a better score. Table \ref{tbl:scores} presents the performance of each of the agents along with the performance of the humans playing against them.

\begin{table}[!h]
\centering
\caption{A comparison between the performance of each of the agents along with the human player who played against each of them.}
\begin{tabular}{| l | p{1.6cm} | p{1.7cm} | p{1.7cm}|}
\hline
\textbf{} & \textbf{Avg. agent's score} & \textbf{Avg. human's score} & \textbf{Avg. social welfare} \\ \hline
\textbf{Careful} & -2.29 & -0.86 & -3.15\\ \hline
\textbf{Aggressive} & -16.27 & -18.40 & -34.67\\ \hline
\textbf{Semi-aggressive} & -60.97 & -62.11 & -123.08\\ \hline
\textbf{Random} & -59.40 & -57.62 & -117.02 \\ \hline
\textbf{Non-Velocity VI} & -6.34 & -9.03 & -15.37\\ \hline
\textbf{Velocity VI} & -5.33 & -6.03 & -11.36\\ \hline
\textbf{Eq. Social VI} & -2.35 & -4.09 & -6.44\\ \hline
\textbf{\agent} & \textbf{15.87} & \textbf{17.12} & \textbf{32.99} \\ \hline
\end{tabular}
\label{tbl:scores}
\end{table}
As shown in Table \ref{tbl:scores}, \agent also significantly outperforms all other agents ($p < 0.01$) in terms of the human's performance. 


In addition, we tested the performance of a velocity value iteration agent with the $\beta$ value set to $0$. That is, an agent that only considers the human reward. Interestingly, such an agent simply moves down and remains there forever. This is because this way it does not disturb the human. Unfortunately, such an agent achieves a final outcome of $-\infty$ (or $-\frac{1}{1-\gamma}$) because it can never reach its destination, since the when the human's reaches her goal, the agent is directly beneath her. 

Next, we evaluate the prediction of the {\it policy evaluation} algorithm,  using both forms of state representations (i.e., with and without velocity). 
Table \ref{tbl:predictions} presents the prediction compared with the actual score of every agent. 
As can be seen in the table, the prediction that uses a state representation with velocity outperforms the prediction that uses a state representation without velocity. However, both predictions performed badly, and imply that our human model is not accurate, as an accurate human model would have resulted with an accurate prediction. 
This demonstrates that it is not enough to rely on the dataset, and strengthens the need for the socially aware approach, which also considers the human's rewards.

\begin{table}[!h]
\centering
\caption{The accuracy of the prediction of a policy evaluation algorithm using a model with velocity and a model without velocity.}
\begin{tabular}{| l | p{1.3cm} | p{1.7cm} | p{1.7cm} | }
\hline
\textbf{} & \textbf{True score} & \textbf{Prediction with velocity (error)} & \textbf{Prediction without velocity (error)}\\ \hline
\textbf{Careful} & -2.29 & -14.41 (12.12) & -4.86 (2.57) \\ \hline
\textbf{Aggressive} & -16.27 & -6.21 (10.6) & 1.14 (17.41) \\ \hline
\textbf{Semi-aggressive} & -60.97 & -56.47 (4.5) & -47.81 (13.16) \\ \hline
\textbf{Non-Velocity VI} & -6.34 & 0.51 (6.85) & \textbf{13.63} (19.97)\\ \hline
\textbf{Velocity VI} & -5.33 & \textbf{14.47} (20.02) & N/A \\ \hline 
\textbf{Eq.Social VI} & -2.35 & 12.34 (14.69) & N/A\\ \hline
\textbf{\agent} & \textbf{15.87} & 7.55 (8.32) & N/A \\ \hline
\end{tabular}
\label{tbl:predictions}
\end{table}

We now turn to analyze the survey results for each agent (see Table \ref{tbl:survey_results}). Each value in the table is the average of all the scores of the measured values: Aggressively, Computer, Generously, Wisely and Predictable.
\begin{table*}
\centering
\caption{Survey results of all agents}
\begin{tabular}{| p{2.6cm} | p{1.4cm} | p{1.4cm} | p{1.4cm} | p{1.4cm} | p{1.4cm} | }
\hline
\textbf{} & \textbf{Aggressively} & \textbf{Computer} & \textbf{Generously} & \textbf{Wisely} & \textbf{Predictable} \\ \hline
\textbf{Careful} & 3.94 & 5.70 & 4.23 & 4.92 & 4.28 \\ \hline
\textbf{Aggressive} & 5.04 & 5.83 & 3.28 & 4.59 & 4.97\\ \hline
\textbf{Semi-aggressive} & 4.57 & 5.73 & 3.21 & 4.33 & 4.52\\ \hline
\textbf{Random} & 3.51 & 5.64 & 4.01 & 3.72 & 3.57\\ \hline
\textbf{Non-Velocity VI} & 4.88 & 6.20 & 3.27 & 4.65 & 4.82\\ \hline
\textbf{Velocity VI} & 4.82 & 6.01 & 4.20 & 4.72 & 4.76\\ \hline
\textbf{Eq. Social VI} & 4.78 & 5.60 & 3.69 & 4.92 & \textbf{4.98}\\ \hline
\textbf{\agent} & \textbf{3.30} & \textbf{5.58} & \textbf{5.14} & \textbf{5.01} & 4.00 \\ \hline
\end{tabular}
\label{tbl:survey_results}
\end{table*}
Note that the lower the `Aggressively' and `Computer' parameters, the better the performance. On the other hand, the higher the `Generously', `Wisely' and `Predictable' parameters, the better the performance.
As can be seen in Table \ref{tbl:survey_results}, \agent obtained the best results compared to the other agents among all parameters except its score on Predictable.
These results entail that \agent demonstrates a clear improvement over all the other agents.

Next, we compare the performance of the humans according to their demographic information. We did not find any significant differences between male and female players, with female participants obtaining an average of $-25.98$ and male participants an average of $-25.66$. Similarly, education level did not seem to have any impact on the performance of the participants. Most interestingly, participants with a driving license that has expired obtained a much lower average score ($-60.15$) than those with a valid driving license ($-24.77$) and those without a driving license. Although these differences appear to be statistically significant using a one-tail t-test ($p<0.05$), this result requires deeper investigation, as the number of participants whose driving license has expired is only $13$. Furthermore, an ANOVA test \cite{rogan1977anova} does not show that these differences are statistically significant.

Finally, we present the number of human participants who followed a strategy that could be in a Nash equilibrium. As can seen in Figure \ref{fig:NE}, only a small portion of the participants followed one of the two strategies that could be in equilibrium: the `Careful' strategy or the `Aggressive'  strategy. Clearly, most of the participants did not follow a strategy that could be in a Nash equilibrium.
\begin{figure}[!h]
\centering
\includegraphics[width=4in]{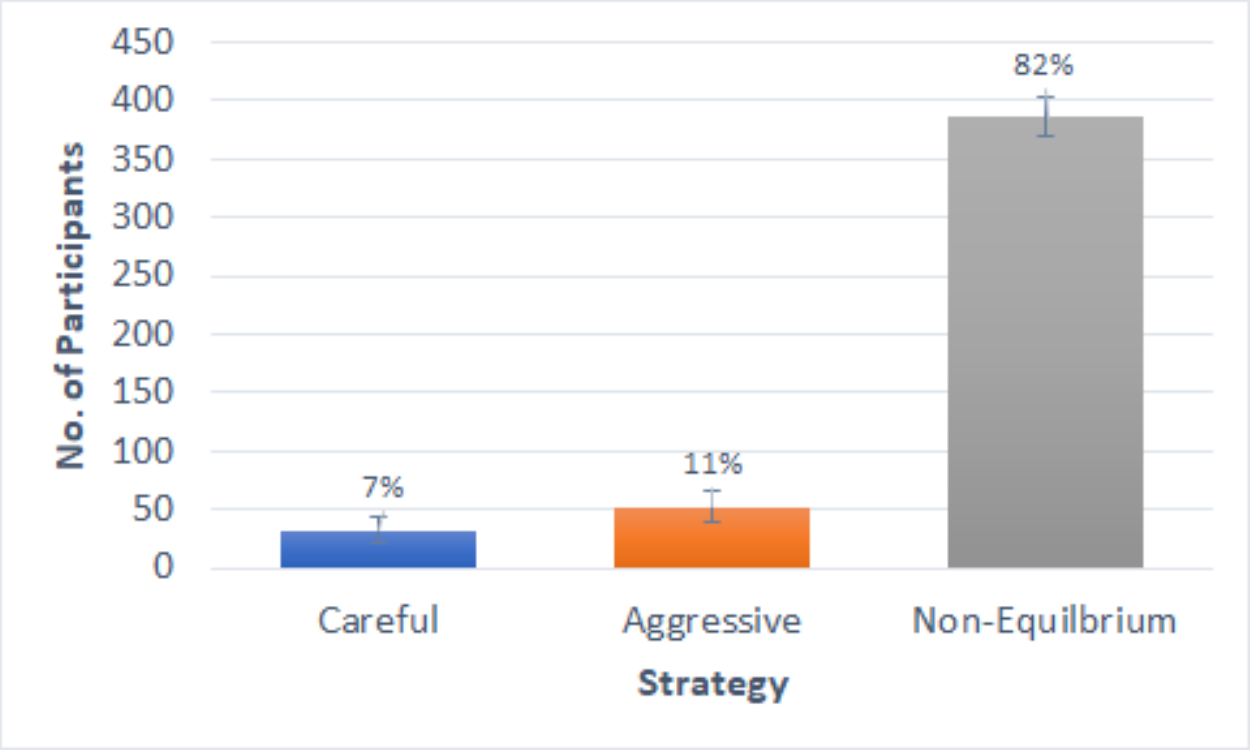}
\caption{The number and percentage of human participants who followed a strategy that could be in a Nash equilibrium as well as the number and percentage of them who did not follow any strategy in equilibrium. The error bars present the 95\% confidence interval.}
\label{fig:NE}
\end{figure}
\section{Conclusions and Future Work}

In this paper we present the single track road problem. In this problem two agents face each-other at opposite positions of a road that can only have one agent pass at a time. We focused on the scenario in which one agent is human, while the other is an autonomous agent. We ran experiments with human subjects in a simple grid domain, which simulates the single track road problem.
We showed that when data is limited,  building an accurate human model is very challenging, and that a reinforcement learning agent, which was based on this data, did not perform well in practice. However, we showed that a social agent, i.e., an agent that tried to maximize a linear combination of the human's utility and its own utility, achieved a high score, and significantly outperformed other baselines, including an agent that tried to maximize only its own utility. We provided a formula to compute what we believe to be a good choice for the $\beta$ parameter, i.e., the ratio between the human's and the agent's utility when attempting to maximize the agent's utility. In addition, we showed that the human achieved highest utility when interacting with \agent, a value that was significantly higher than when interacting with any other baseline agent. Furthermore, \agent was perceived by humans as less aggressive, more generous and wiser than all other baselines.

In future work we intend to show that \agent performs well also when considering other, possibly very different, settings. One option for such a setting is a setting with a continuous state space as well as a continuous action space.
We further intend  to utilise the idea of using a social agent approach, learned in the grid game environment, and to adapt \agent to a simulated autonomous vehicle environment with human drivers controlling simulated vehicles. Once we perform well in the simulated environment, we expect to run \agent in a real single-track-road scenario, with an autonomous vehicle and human drivers. We hope to show that \agent will perform well in the real-word environment, and that a social agent approach will be useful in practice.
Another direction for future work is to focus on situations in which the human reward function is not available apriori. Such a situation would challenge the use of \agent, as it uses the human reward function for computing its objective function.
One appealing option may be to use inverse reinforcement learning \cite{ng2000algorithms} to first learn the human's reward function, and then, to use this function to compute the optimal policy for \agent.


\section*{Acknowledgment}

This research was supported in part by the Ministry of Science, Technology \& Space, Israel.



%
\bibliographystyle{abbrv}
\bibliography{bibliography}

\end{document}